\documentclass[10pt,twocolumn,letterpaper]{article}
\usepackage[top=0.7in,bottom=0.7in,left=0.6in,right=0.7in,columnsep=0.35in]{geometry}
\date{\vspace{-0.3in}}

\usepackage{ifpdf}

\usepackage{cite}

  \usepackage[pdftex]{graphicx}
\graphicspath{{./images/}}

\usepackage{amsmath}

\usepackage{algorithmic}

\usepackage{array}

  \usepackage[caption=false,font=footnotesize]{subfig}

\usepackage{fixltx2e}

\usepackage{stfloats}

\usepackage{url}

\usepackage{indentfirst,graphicx}
\usepackage{amssymb,bm,bbm}
\usepackage{amsthm}
\newtheorem{theorem}{Theorem}
\newtheorem{corollary}[theorem]{Corollary}
\newtheorem{lemma}[theorem]{Lemma}
\newtheorem{condition}{Condition}
\newcommand{\sss}{\hspace{-0.015in}}

\begin{document}
\title{Information-Theoretic Lower Bounds for \\ Recovery of Diffusion Network Structures}

\author{
Keehwan Park\\
Department of Computer Science\\
Purdue University\\
park451@purdue.edu
\and
Jean Honorio\\
Department of Computer Science\\
Purdue University\\
jhonorio@purdue.edu
}

\maketitle

\begin{abstract}
We study the information-theoretic lower bound of the sample complexity of the correct recovery of diffusion network structures. We introduce a discrete-time diffusion model based on the Independent Cascade model for which we obtain a lower bound of order $\Omega(k \log p)$, for directed graphs of $p$ nodes, and at most $k$ parents per node. Next, we introduce a continuous-time diffusion model, for which a similar lower bound of order $\Omega(k \log p)$ is obtained. Our results show that the algorithm of \cite{pouget2015inferring} is statistically optimal for the discrete-time regime. Our work also opens the question of whether it is possible to devise an optimal algorithm for the continuous-time regime.
\end{abstract}

\section{Introduction}
In recent years, the increasing popularity of online social network services, such as Facebook, Twitter, and Instagram, allows researchers to access large influence propagation traces. Since then, the influence diffusion on social networks has been widely studied in the data mining and machine learning communities. Several studies showed how influence propagates in such social networks as well as how to exploit this effect efficiently. Domingos et al. \cite{domingos2001mining} first explored the use of social networks in viral marketing. Kempe et al. \cite{kempe2003maximizing} proposed the influence maximization problem on the Independent Cascade (IC) and Linear Threshold (LT) models, assuming all influence probabilities are known. \cite{Saito:2008:PID:1430307.1430318,goyal2010learning} studied the learning of influence probabilities for a known (fixed) network structure. \par
The network inference problem consists in discovering the underlying functional network from cascade data. The problem is particularly important since regardless of having some structural side information, e.g., friendships in online social networks, the functional network structure, which reflects the actual influence propagation paths, may look greatly different. Adar et al. \cite{1517844} first explored the problem of inferring the underlying diffusion network structure. The subsequent researches \cite{myers2010convexity,gomez2010inferring} have been done in recent years and the continuous-time extensions \cite{saito2009learning,gomez11netrate,du13nips} have also been explored in depth. \par

\textbf{Basic diffusion model. }
Consider a directed graph, $\mathcal{G}=(\mathcal{V},\mathcal{E})$ where $\mathcal{V}=\{1,\ldots,p\}$ is the set of nodes and $\mathcal{E}$ is the set of edges. Next, we provide a short description for the discrete-time IC model \cite{kempe2003maximizing}. Initially we draw an initial set of active nodes from a source distribution. The process unfolds in discrete steps. When node $j$ first becomes active at time $t$, it independently makes a single attempt to activate each of its outgoing, inactive neighbors $i$, with probability $\theta_{j,i}$. If $j$ succeeds, then $i$ will become active at time $t+1$. If $j$ fails, then it makes no further attempts to activate $i$. And this process runs until no more activations are possible. \par

\textbf{Related works. }
Research on the sample complexity of the network inference problem is very recent \cite{netrapalli2012learning,abrahao2013trace,daneshmand2014estimating,narasimhan2015learnability,pouget2015inferring}. 
Netrapalli et al. \cite{netrapalli2012learning} studied the network inference problem based on the discrete-time IC model and showed that for graphs of $p$ nodes and at most $k$ parents per node,  $\mathcal{O}(k^2 \log p)$ samples are sufficient, and $\Omega(k \log p)$ samples are necessary. However, as Daneshmand et al.\cite{daneshmand2014estimating} have pointed out, their model only considers the discrete-time diffusion model and the correlation decay condition is rather restrictive since it limits the number of new activations at every step. Abrahao et al. \cite{abrahao2013trace} proposed the First-Edge algorithm to solve the network inference problem and also suggested lower bounds but their results are specific to their algorithm, i.e., the lower bounds are not information-theoretic. \par
In \cite{daneshmand2014estimating}, Daneshmand et al. worked on the continuous-time network inference problem with $\ell$-1 regularized maximum likelihood estimation and showed that $\mathcal{O}(k^3 \log p)$ samples are sufficient, using the \textit{primal-dual witness} method. Narasimhan et al. \cite{narasimhan2015learnability} explored various influence models including IC, LT, and Voter models under the Probably Approximately Correct learning framework. Pouget-Abadie et al. \cite{pouget2015inferring} studied various discrete-time models under the restricted eigenvalue conditions. They also proposed the first algorithm which recovers the network structure with high probability in $\mathcal{O}(k \log p)$ samples. \par
It is important to note that, as we will see later in the paper, we show information-theoretic lower bounds of order $\Omega(k \log p)$, confirming that the algorithm in \cite{pouget2015inferring} is statistically optimal. However, since their algorithm only considered discrete-time models, developing a new algorithm for continuous-time models with the sufficient condition on the sample complexity of order $\mathcal{O} (k \log p)$ can be an interesting future work.

\section{Ensemble of Discrete-time Diffusion Networks}
Lower bounds of the sample complexity for general graphs under the IC and LT models \cite{kempe2003maximizing} seem to be particularly difficult to analyze. In this paper, we introduce a simple network under IC model, which fortunately allow us to show sample complexity lower bounds that match the upper bounds found in \cite{pouget2015inferring} for discrete-time models.
\subsection{A simple two-layer network}
Here we considered the two-layer IC model shown in Figure~\ref{fig_simple_model}. Although not realistic, the considered model allows to show that even in this simple two-layer case, we require $\Omega(k \log p)$ samples in order to avoid network recovery failure.

In Figure~\ref{fig_simple_model}, each circle indicates a node and each edge $(j,i)$ with its influence probability $\theta$ indicates that a cascade can be propagated from node $j$ to $i$ or equivalently node $j$ activates $i$ with probability $\theta$. The model assumes that there exists a super source node $s_1$, which is already activated at time zero and at time $1$, it independently tries to activate $p$ parent nodes with probability $\theta_0$ and $s_2$ with probability 1. There exist a child node $p+1$, which has exactly $k+1$ parents including $s_2$. Then at time $2$, $s_2$ and all direct parents of $p+1$, which have been activated at time $1$, independently try to activate the child node $p+1$ with probability $\theta_0$ and $\theta$, respectively. We use $t_i=\infty$ to indicate that a node $i$ has not been activated during the cascading process. Note that these influence probabilities can be generalized without too much effort. \par

Given the model with unknown edges between parent nodes and the child node $p+1$, and a set of $n$ samples $\bm{t^{(1)}}, \bm{t^{(2)}}, \ldots, \bm{t^{(n)}} \in \{1,\infty\}^p \times \{2,\infty\}$, the goal of the learner is to recover the $k$ edges or equivalently to identify the $k \ll p$ direct parents of the child node $p+1$. Each sample is a ($p+1$)-dimensional vector, $\bm{t}=(t_1, \ldots, t_p, t_{p+1})$, and includes all the activation times of the parent and child nodes. A parent node $i \in \{1,\ldots,p\}$ is either activated at time 1 (i.e., $t_i = 1$) or not (i.e.,$ t_i = \infty$). The child node $p+1$ is either activated at time 2 (i.e., $t_{p+1}=2$) or not (i.e., $t_{p+1}=\infty$).

Now, we define the hypothesis class $\mathcal{F}$ as the set of all combinations of $k$ nodes from $p$ possible parent nodes, that is $|\mathcal{F}| := {p \choose k}$. Thus, a hypothesis $\pi$ is the set of $k$ parent nodes such that $\forall i \in \pi$, there exist an edge from $i$ to $p+1$ with influence probability $\theta$. We also let $\pi^c := \{1, \ldots, p\} \backslash \pi$ to be the complement set of $\pi$.
Given a hypothesis $\pi$ and a sample $\bm{t}$, we can write a data likelihood using independence assumptions.
\begin{align}
\mathbb{P}(\bm{t} ; \pi) & = \mathbb{P}(\bm{t}_{\pi}) \mathbb{P}(\bm{t}_{\pi^c}) \mathbb{P}(t_{p+1} | \bm{t}_{\pi} t_{s_2}) \label{eq_likelihood}
\end{align}
The conditional probability can be expressed as follows.
\begin{align*}
\mathbb{P}(t_{p+1} = 2 | \bm{t}_{\pi} t_{s_2}) & = 1-(1-\theta)^{\sum_{i \in \pi} \mathbbm{1}[t_i=1]}(1-\theta_0)\\
\mathbb{P}(t_{p+1} = \infty | \bm{t}_{\pi} t_{s_2}) & = (1-\theta)^{\sum_{i \in \pi} \mathbbm{1}[t_i=1]}(1-\theta_0)
\end{align*}
where $\mathbbm{1}[\cdot]$ is an indicator function. Lastly, for simplicity, we define,
\begin{align}
\theta := 1-\theta_0^{\frac{1}{k}} \label{eq_theta_def}
\end{align}
which decreases as the child node $p+1$ has more parents. The latter agrees with the intuition that as we have more parents, the chance of a single parent activating the child node gets smaller.
\par

We will study the information-theoretic lower bounds on the sample complexity of the network inference problem. We will use Fano's inequality in order to analyze the necessary number of samples for any conceivable algorithm in order to avoid failure.

\begin{figure}[t!]
    \centering
        \includegraphics[width=2.5in]{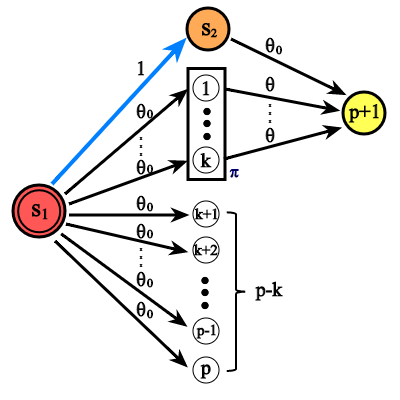}
    \caption{Diffusion Model with Two Layers.}
    \label{fig_simple_model}
\end{figure}

\subsection{Lower Bounds with Fano's inequality}
First, we will bound the mutual information by using a pairwise Kullback-Leibler (KL) divergence-based bound \cite{Yu97}, and show the following lemma.
\begin{lemma}
\label{lemma_kl_bound1}
Under the settings of the discrete-time diffusion model, for any pair of hypotheses, $\pi, \pi' \in \mathcal{F}$,
\begin{align*}
\mathbb{KL}(\mathcal{P}_{\bm{t}|\pi} || \mathcal{P}_{\bm{t}|\pi'}) \le \log\frac{1}{\theta_0}
\end{align*}
\end{lemma}
\begin{proof}
First, we notice that the maximum KL divergence between two distributions, $\mathcal{P}_{\bm{t}|\pi}$ and $\mathcal{P}_{\bm{t}|\pi'}$ can be achieved when the two sets, $\pi$ and $\pi'$, do not share any node, or equivalently, when there is not any overlapping edge between parent and child nodes. That is, $\pi \cap \pi' = \emptyset$. \par
Then we compute the KL divergence with the two disjoint parent sets, as follows
\begin{align*}
\mathbb{KL}(\mathcal{P}_{\bm{t}|\pi} || \mathcal{P}_{\bm{t}|\pi'}) & = \sum_{\bm{t} \in \{1,\infty\}^p \times \{2,\infty\}} \mathbb{P}(\bm{t}|\pi) \log \frac{\mathbb{P}(\bm{t}|\pi)}{\mathbb{P}(\bm{t}|\pi')}
\end{align*}
Using Jensen's inequality and Eq~(\ref{eq_likelihood}), we have
\begin{align}
\mathbb{KL}(& \mathcal{P}_{\bm{t}|\pi} || \mathcal{P}_{\bm{t}|\pi'}) \le \log \bigg( \sum_{\bm{t} \in \{1,\infty\}^p \times \{2,\infty\}} \mathbb{P}(\bm{t}|\pi) \frac{\mathbb{P}(\bm{t}|\pi)}{\mathbb{P}(\bm{t}|\pi')} \bigg) \nonumber \\
& \le \log \bigg( \max_{\bm{t} \in \{1,\infty\}^p \times \{2,\infty\}} \frac{\mathbb{P}(\bm{t}|\pi)}{\mathbb{P}(\bm{t}|\pi')} \bigg) \nonumber\\
& = \log \Bigg( \max_{\bm{t} \in \{1,\infty\}^p \times \{2,\infty\}} \frac{\mathbb{P}(\bm{t}_{\pi}) \mathbb{P}(\bm{t}_{\pi^c}) \mathbb{P}(t_{p+1} | \bm{t}_{\pi} t_{s_2})}{\mathbb{P}(\bm{t}_{\pi'}) \mathbb{P}(\bm{t}_{\pi'^c}) \mathbb{P}(t_{p+1} | \bm{t}_{\pi'} t_{s_2})} \Bigg) \nonumber\\
& = \log \Bigg( \max_{\bm{t} \in \{1,\infty\}^p \times \{2,\infty\}} \frac{\mathbb{P}(t_{p+1} | \bm{t}_{\pi} t_{s_2})}{\mathbb{P}(t_{p+1} | \bm{t}_{\pi'} t_{s_2})} \Bigg) \label{eq_main1}
\end{align}
Now as we have argued earlier, the maximum value can be attained when $\pi \cap \pi' = \emptyset$. Without loss of generality, we assume that $\pi$ connects the first $k$ nodes to $p+1$ and $\pi'$ connects the subsequent $k$ nodes to $p+1$.
Thus we have
\begin{align*}
\frac{\mathbb{P}(t_{p+1}=2 | \bm{t}_{\pi} t_{s_2})}{\mathbb{P}(t_{p+1}=2 | \bm{t}_{\pi'} t_{s_2})} \le \frac{1-(1-\theta)^{\sum_{i=1}^{k} \mathbbm{1}[t_i=1]}(1-\theta_0)}{1-(1-\theta)^{\sum_{i=k+1}^{2k} \mathbbm{1}[t_i=1]}(1-\theta_0)}
\end{align*}
Similarly, we have
\begin{align*}
\frac{\mathbb{P}(t_{p+1}=\infty | \bm{t}_{\pi} t_{s_2})}{\mathbb{P}(t_{p+1}=\infty | \bm{t}_{\pi'} t_{s_2})} \le \frac{(1-\theta)^{\sum_{i=1}^{k} \mathbbm{1}[t_i=1]}(1-\theta_0)}{(1-\theta)^{\sum_{i=k+1}^{2k} \mathbbm{1}[t_i=1]}(1-\theta_0)}
\end{align*}
We can use the above expressions in order to obtain an upper bound for Eq~(\ref{eq_main1}).
Thus, by Eq~(\ref{eq_theta_def}) we have
\begin{align*}
& \mathbb{KL}( \mathcal{P}_{\bm{t}|\pi} || \mathcal{P}_{\bm{t}|\pi'}) \\
& \le \log \Bigg( \max \bigg\{ \frac{1-(1-\theta)^k (1-\theta_0)}{\theta_0}, \frac{1-\theta_0}{(1-\theta)^k(1-\theta_0)} \bigg\} \Bigg) \\
& \le \log \bigg( \frac{1}{\theta_0} \bigg)
\end{align*}
\end{proof}

By using the above results, we show that the necessary number of samples for the network inference problem is $\Omega(k \log p)$.
\begin{theorem}
\label{thm_fano1}
Suppose that nature picks a ``true'' hypothesis $\bar\pi$ uniformly at random from some distribution of hypotheses with support $\mathcal{F}$. Then a dataset $S$ of $n$ independent samples $\bm{t^{(1)}}, \bm{t^{(2)}}, \ldots, \bm{t^{(n)}} \in \{1,\infty\}^p \times \{2,\infty\}$ is produced, conditioned on the choice of $\bar\pi$. The learner then infers $\hat\pi$ from the dataset $S$. Under the settings of the two-layered discrete-time diffusion model, there exists a network inference problem of $k$ direct parent nodes such that if $n \le \frac{k\log{p} -k\log{k} - 2\log{2}}{2 \log \frac{1}{\theta_0}}$, then learning fails with probability at least $1/2$, i.e.,
\begin{align*}
\mathbb{P}[\hat\pi \ne \bar\pi] \ge \frac{1}{2}
\end{align*}
for any algorithm that a learner could use for picking $\hat\pi$.
\end{theorem}
\begin{proof}
We first bound the mutual information by the pairwise KL-based bound \cite{Yu97}.
\begin{align*}
\mathbb{I}(\bar\pi, S) & < \frac{1}{|\mathcal{F}|^2}\sum_{\pi \in \mathcal{F}} \sum_{\pi' \in \mathcal{F}} \mathbb{KL}(\mathcal{P}_{S|\pi} || \mathcal{P}_{S|\pi'})\\
& = \frac{n}{|\mathcal{F}|^2} \sum_{\pi \in \mathcal{F}} \sum_{\pi' \in \mathcal{F}} \mathbb{KL}(\mathcal{P}_{\bm{t}|\pi} || \mathcal{P}_{\bm{t}|\pi'})
\end{align*}
Now from Lemma~\ref{lemma_kl_bound1}, we can bound the mutual information as follows.
\begin{align}
\mathbb{I}(\bar\pi, S) < n \log\frac{1}{\theta_0} \label{eq_mut_info1}
\end{align}
Finally, by Fano's inequality~\cite{Cover06}, Eq~(\ref{eq_mut_info1}), and the well-known bound, $\log{p \choose k} \ge k(\log{p} - \log{k})$, we have
\begin{align*}
&& \mathbb{P}[\hat{f} \ne \bar{f}] & \ge 1- \frac{n \log \frac{1}{\theta_0} + \log{2}}{\log{p \choose k}} \\
&& & \ge 1- \frac{n \log \frac{1}{\theta_0} + \log{2}}{k(\log{p} - \log{k})} \\
&& & = \frac{1}{2}
\end{align*}
By solving the last equality we conclude that, if $n \le \frac{k\log{p} -k\log{k} - 2\log{2}}{2 \log \frac{1}{\theta_0}}$, then any conceivable algorithm will fail with a large probability, $\mathbb{P}[\hat{\pi} \ne \bar{\pi}] \ge 1/2$.
\end{proof} \par

\section{Ensemble of Continuous-time Diffusion Networks}
In this section, we will study the continuous-time extension to the two-layer diffusion model. For this purpose, we introduce a transmission function between parent and child nodes. For the interested readers, Gomez-Rodriguez et al. \cite{gomez11netrate} discuss transmission functions in full detail.
\subsection{A simple two-layer network}
Here we used the same two-layer network structure shown in Figure~\ref{fig_simple_model}. However, for a general continuous model, the activation time for a child node is dependent on the activation times of its parents. For our analysis, we relax this assumption by considering a fixed time range for each layer. In other words, we first consider a fixed time span, $T$. Then the $p$ parent nodes are only activated between $[0,T]$, and the child node $p+1$ is only activated between $[T, 2T]$. Our analysis for the continuous-time model largely borrows from our understanding of the discrete-time model. \par
The continuous-time model works as follows. The super source node $s_1$, tries to activate each of the $p$ parent nodes with probability $\theta_0$, and $s_2$ with probability 1. If a parent node gets activated, it picks an activation time from $[0,T]$ based on the transmission function, $f(t; \pi)$. Then, $s_2$ and all the direct parents, which have been activated in $t \in [0,T]$, independently try to activate the child node $p+1$ with probability $\theta_0$ and $\theta$, respectively. If the child node $p+1$ gets activated, it picks an activation time from $[T, 2T]$ based on the transmission function, $f(t; \pi)$. \par
For the continuous-time model, the conditional probabilities can be expressed as follows.
\begin{align*}
\mathbb{P}(t_{p+1} & \in [T,2T] | \bm{t}_{\pi} t_{s_2}) =\\
& \bigg( 1-(1-\theta)^{\sum_{i \in \pi} \mathbbm{1}[t_i \in [0,T]]}(1-\theta_0) \bigg) \cdot f(t_{p+1}-T; \pi) \\
\mathbb{P}(t_{p+1} & = \infty | \bm{t}_{\pi} t_{s_2}) = (1-\theta)^{\sum_{i \in \pi} \mathbbm{1}[t_i \in [0,T]]}(1-\theta_0)
\end{align*}
Lastly, we define the domain of a sample $\bm{t}$ to be $\mathcal{T}~:=~([0,T]~\cup~\{\infty\})^p \times ([T,2T]~\cup~\{\infty\})$.

\subsection{Boundedness of Transmission Functions}
We will start with the general boundedness of the transmission functions. The constants in the boundedness condition will be later directly related to the lower bound of the sample complexity. In the later part of the paper, we will provide an example for the exponentially distributed transmission function. Often, transmission functions used in the literature fulfill this assumption, e.g., the Rayleigh distribution \cite{daneshmand2014estimating} and the Weibull distribution for $\mu \ge 1$ \cite{kurashima2014probabilistic}.
\begin{condition}[Boundedness of transmission functions]
\label{boundedness_cond}
Suppose $t \in [0,T]$ is a transmission time random variable, dependent on its parents $\pi$. The probability density function $f(t; \pi)$ fulfills the following condition for a pair of positive constants $\kappa_1$ and $\kappa_2$.
\begin{align*}
\min_{t \in [0,T]} f(t; \pi) \ge \kappa_1 > 0\\
\max_{t \in [0,T]} f(t; \pi) \le \kappa_2 < \infty
\end{align*}
\end{condition}

\subsection{Lower Bounds with Fano's inequality}
First, we provide a bound on the KL divergence that will be later used in analyzing the necessary number of samples for the network inference problem.
\begin{lemma}
\label{lemma_kl_bound2}
Under the settings of the continuous-time diffusion model, for any pair of hypotheses, $\pi, \pi' \in \mathcal{F}$,
\begin{align*}
\mathbb{KL}(\mathcal{P}_{\bm{t}|\pi} || \mathcal{P}_{\bm{t}|\pi'}) \le \log \Bigg( \max \bigg\{ \frac{\kappa_2}{\kappa_1} \bigg( \frac{1}{\theta_0}-(1-\theta_0) \bigg), \frac{1}{\theta_0} \bigg\} \Bigg)
\end{align*}
\end{lemma}
\begin{proof}
We note that the proof is very similar to that of Lemma~\ref{lemma_kl_bound1}.
\begin{align}
\mathbb{KL}(\mathcal{P}_{\bm{t}|\pi} || \mathcal{P}_{\bm{t}|\pi'}) & = \sum_{\bm{t} \in \mathcal{T}} \mathbb{P}(\bm{t}|\pi) \log \frac{\mathbb{P}(\bm{t}|\pi)}{\mathbb{P}(\bm{t}|\pi')} \nonumber \\
& \le \log \bigg( \max_{\bm{t} \in \mathcal{T}} \frac{\mathbb{P}(\bm{t}|\pi)}{\mathbb{P}(\bm{t}|\pi')} \bigg) \nonumber\\
& = \log \Bigg( \max_{\bm{t} \in \mathcal{T}} \frac{\mathbb{P}(t_{p+1} | \bm{t}_{\pi} t_{s_2})}{\mathbb{P}(t_{p+1} | \bm{t}_{\pi'} t_{s_2})} \Bigg) \label{eq_main2}
\end{align}
Now with the same argument we made in Lemma~\ref{lemma_kl_bound1}, consider that $\pi$ connects the first $k$ nodes to $p+1$ and $\pi'$ connects the subsequent $k$ nodes to $p+1$.
Thus, we have
\begin{align*}
& \frac{\mathbb{P}(t_{p+1} \in [T,2T] | \bm{t}_{\pi} t_{s_2})}{\mathbb{P}(t_{p+1} \in [T,2T] | \bm{t}_{\pi'} t_{s_2})} \le \\
&\ \ \ \  \frac{\bigg( 1-(1-\theta)^{\sum_{i=1}^{k} \mathbbm{1}[t_i \in [0,T]]}(1-\theta_0) \bigg) f(t_{p+1}-T; \pi)}{\bigg( 1-(1-\theta)^{\sum_{i=k+1}^{2k} \mathbbm{1}[t_i \in [0,T]]}(1-\theta_0) \bigg) f(t_{p+1}-T; \pi')}
\end{align*}
Similarly, we have
\begin{align*}
\frac{\mathbb{P}(t_{p+1}=\infty | \bm{t}_{\pi} t_{s_2})}{\mathbb{P}(t_{p+1}=\infty | \bm{t}_{\pi'} t_{s_2})} \le \frac{(1-\theta)^{\sum_{i=1}^{k} \mathbbm{1}[t_i \in [0,T]]}(1-\theta_0)}{(1-\theta)^{\sum_{i=k+1}^{2k} \mathbbm{1}[t_i \in [0,T]]}(1-\theta_0)}
\end{align*}
We can use the above expressions in order to obtain an upper bound for Eq~(\ref{eq_main2}).
Thus, by Eq~(\ref{eq_theta_def}) we have
\begin{align*}
&\mathbb{KL}(\mathcal{P}_{\bm{t}|\pi} || \mathcal{P}_{\bm{t}|\pi'}) \\
& \le \log \sss\Bigg( \sss\max \bigg\{\sss\frac{1-(1-\theta)^k (1-\theta_0)}{\theta_0} \frac{\kappa_2}{\kappa_1}, \frac{1-\theta_0}{(1-\theta)^k(1-\theta_0)}\sss\bigg\} \sss\Bigg) \\
& = \log \sss\Bigg( \sss\max \bigg\{ \sss\frac{\kappa_2}{\kappa_1} \bigg( \frac{1}{\theta_0}-(1-\theta_0) \bigg), \frac{1}{\theta_0} \sss\bigg\} \sss\Bigg)
\end{align*}
\end{proof}

By using the above results, we show that the necessary number of samples for the network inference problem is also $\Omega(k \log p)$ in the continuous-time model.
\begin{theorem}
\label{thm_fano2}
Suppose that nature picks a ``true'' hypothesis $\bar\pi$ uniformly at random from some distribution of hypotheses with support $\mathcal{F}$. Then a dataset $S$ of $n$ independent samples $\bm{t^{(1)}}, \bm{t^{(2)}}, \ldots, \bm{t^{(n)}} \in ([0,T] \cup \{\infty\})^p \times ([T,2T] \cup \{\infty\})$ is produced, conditioned on the choice of $\bar\pi$. The learner then infers $\hat\pi$ from the dataset $S$. Assume that the transmission function $f(t; \pi)$, satisfies Condition~\ref{boundedness_cond} with constants $\kappa_1$ and $\kappa_2$. Under the settings of the two-layered continuous-time diffusion model, there exists a network inference problem of $k$ direct parent nodes such that if 
\begin{align*}
n \le \frac{k\log{p} - k\log{k} - 2\log{2}}{2 \log \Bigg( \max \bigg\{ \frac{\kappa_2}{\kappa_1} \bigg( \frac{1}{\theta_0}-(1-\theta_0) \bigg), \frac{1}{\theta_0} \bigg\} \Bigg)}
\end{align*}
then learning fails with probability at least $1/2$, i.e.,
\begin{align*}
\mathbb{P}[\hat\pi \ne \bar\pi] \ge \frac{1}{2}
\end{align*}
for any algorithm that a learner could use for picking $\hat\pi$.
\end{theorem}
\begin{proof}
The proof is very similar to that of Theorem~\ref{thm_fano1}.
First, by the pairwise KL-based bound \cite{Yu97} and Lemma~\ref{lemma_kl_bound2}, we have
\begin{align}
\mathbb{I}(\bar\pi, S) < n \log \Bigg( \max \bigg\{ \frac{\kappa_2}{\kappa_1} \bigg( \frac{1}{\theta_0}-(1-\theta_0) \bigg), \frac{1}{\theta_0} \bigg\} \Bigg) \label{eq_mut_info2}
\end{align}
By Fano's inequality~\cite{Cover06}, Eq~(\ref{eq_mut_info2}), and the well-known bound, $\log{p \choose k} \ge k(\log{p} - \log{k})$, we have
\begin{align*}
&\hspace{-0.1in}\mathbb{P} [\hat{f} \ne \bar{f}] \\
& \ge 1- \frac{n \log \Bigg( \max \bigg\{ \frac{\kappa_2}{\kappa_1} \bigg( \frac{1}{\theta_0}-(1-\theta_0) \bigg), \frac{1}{\theta_0} \bigg\} \Bigg) + \log{2}}{\log{p \choose k}} \\
& \ge 1- \frac{n \log \Bigg( \max \bigg\{ \frac{\kappa_2}{\kappa_1} \bigg( \frac{1}{\theta_0}-(1-\theta_0) \bigg), \frac{1}{\theta_0} \bigg\} \Bigg) + \log{2}}{k(\log{p} - \log{k})} \\
& = \frac{1}{2}
\end{align*}
By solving the last equality we conclude that, if $n \le \frac{k\log{p} - k\log{k} - 2\log{2}}{2 \log \Bigg( \max \bigg\{ \frac{\kappa_2}{\kappa_1} \bigg( \frac{1}{\theta_0}-(1-\theta_0) \bigg), \frac{1}{\theta_0} \bigg\} \Bigg)}$, then any conceivable algorithm will fail with a large probability, $\mathbb{P}[\hat{\pi}~\ne~\bar{\pi}] \ge 1/2$.
\end{proof} \par

Lastly, we will present an example for the exponentially distributed transmission function.
\begin{corollary}[Exponential Distribution]
Suppose that nature picks a ``true'' hypothesis $\bar\pi$ uniformly at random from some distribution of hypotheses with support $\mathcal{F}$. Then a dataset $S$ of $n$ independent samples $\bm{t^{(1)}}, \bm{t^{(2)}}, \ldots, \bm{t^{(n)}} \in ([0,T] \cup \{\infty\})^p \times ([T,2T] \cup \{\infty\})$ is produced, conditioned on the choice of $\bar\pi$. The learner then infers $\hat\pi$ from the dataset $S$. Assume that the transmission function $f(t; \pi) = \frac{\lambda e^{-\lambda t}}{1 - e^{-\lambda T}}$ is of the censored (rescaled) exponential distribution form, defined over [0,T]. Under the settings of the two-layered continuous-time diffusion model, there exists a network inference problem of $k$ direct parent nodes such that if
\begin{align*}
n \le \frac{k\log{p} - k\log{k} - 2\log{2}}{2 \log \Bigg( \max \bigg\{ e^{\lambda T} \bigg( \frac{1}{\theta_0}-(1-\theta_0) \bigg), \frac{1}{\theta_0} \bigg\} \Bigg)}
\end{align*}
then learning fails with probability at least $1/2$, i.e.,
\begin{align*}
\mathbb{P}[\hat\pi \ne \bar\pi] \ge \frac{1}{2}
\end{align*}
for any algorithm that a learner could use for picking $\hat\pi$.
\end{corollary}
\begin{proof}
Since the probability density function should only be defined between $[0,T]$, we need to rescale the probability density function of the standard exponential distribution, $g(t)~\sim~Exp(\lambda)$, whose cumulative density function is $G(t)$. Given this, we have the censored (rescaled) transmission function,
\begin{align*}
f(t; \pi) = \frac{g(t)}{G(T) - G(0)} = \frac{g(t)}{G(T)} = \frac{\lambda e^{-\lambda t}}{1 - e^{-\lambda T}}
\end{align*}
From the above, we can obtain the minimum and maximum values of the density function, $\kappa_1$ and $\kappa_2$, in Condition~\ref{boundedness_cond} as follows.
\begin{align}
\kappa_1 = \frac{\lambda e^{-\lambda T}}{1 - e^{-\lambda T}}&&, && \kappa_2 = \frac{\lambda}{1 - e^{-\lambda T}} && \Rightarrow && \frac{\kappa_2}{\kappa_1} = e^{\lambda T} \label{eq_exponential_kappa}
\end{align}
Finally using Theorem~\ref{thm_fano2} and Eq~(\ref{eq_exponential_kappa}), we show that if
\begin{align*}
n \le \frac{k\log{p} - k\log{k} - 2\log{2}}{2 \log \Bigg( \max \bigg\{ e^{\lambda T} \bigg( \frac{1}{\theta_0}-(1-\theta_0) \bigg), \frac{1}{\theta_0} \bigg\} \Bigg)}
\end{align*}
then any conceivable algorithm will fail with a large probability, $\mathbb{P}[\hat{\pi} \ne \bar{\pi}] \ge 1/2$.
\end{proof}

\section{Conclusion}
We have formulated the two-layered discrete-time and continuous-time diffusion models and derived the information-theoretic lower bounds of the sample complexity of order $\Omega (k \log p)$. Our bound is particularly important since we can infer that the algorithm in \cite{pouget2015inferring}, which only works under discrete-time settings, is statistically optimal based on our bound. \par
Our work opens the question of whether it is possible to devise an algorithm for which the sufficient number of samples is $\mathcal{O}(k \log p)$ in continuous-time settings. We also have observed some potential future work to analyze sharp phase transitions for the sample complexity of the network inference problem.

\bibliography{./mybib}

\end{document}